\pdfoutput=1
\documentclass[letterpaper,10pt,conference]{ieeeconf}
\IEEEoverridecommandlockouts
\usepackage{amsmath,amssymb}
\usepackage{booktabs}
\usepackage{graphicx}
\usepackage{bm}
\usepackage[breaklinks,colorlinks,linkcolor=black,citecolor=black,urlcolor=black]{hyperref}
\hypersetup{pdftitle={Algebraic Consistency Alone Does Not Certify Temporal Structure in Latent Action Models},
            pdfauthor={Di Wen, Ruodi Zhang, Kailun Yang, Kunyu Peng},
            pdfkeywords={latent action models, learning from video, evaluation, representation learning, robot manipulation}}
\newtheorem{proposition}{Proposition}

\title{Algebraic Consistency Alone Does Not Certify Temporal Structure in Latent Action Models}
\author{Di Wen$^{1}$, Ruodi Zhang$^{1}$, Kailun Yang$^{2}$, and Kunyu Peng$^{1,\dagger}$%
\thanks{$^{1}$Di Wen, Ruodi Zhang and Kunyu Peng are with the Karlsruhe Institute of Technology, Karlsruhe 76131, Germany.}%
\thanks{$^{2}$Kailun Yang is with Hunan University, Changsha 410012, China.}%
\thanks{First author (email: di.wen@kit.edu).}%
\thanks{$^{\dagger}$Corresponding author (email: kunyu.peng@kit.edu).}%
\thanks{This work has been submitted to the IEEE for possible publication. Copyright may be transferred without notice, after which this version may no longer be accessible.}%
}

\begin{document}
\bstctlcite{IEEEexample:BSTcontrol}
\maketitle
\begin{abstract}
Latent action models infer a code for the transition between two frames of action-free video. Recent methods regularise this code to compose additively and reverse antisymmetrically, and report order-of-magnitude reductions in the resulting errors as a label-free certificate that the code has captured temporal structure. We show that this conclusion does not follow. Reconstruction drives the decoded transition toward a difference of state features, for which both identities hold for any pairing, a solution the metric cannot distinguish from one encoding nuisance state or a coordinate convention. Across five source domains, a trained but unconstrained counterpart already achieves $83$--$97\%$ of the reduction relative to an untrained anchor. The residual fold is governed as much by the decoder family as by what is learned. A constrained model retrained after its temporal pairing is destroyed still reaches, in each domain, a lower error than the unconstrained model on real data. Downstream, preserving the temporal pairing yields no consistent advantage on LIBERO-GOAL or LIBERO-SPATIAL, and across the tested arms the code's mean linear action decodability falls as the algebraic error improves. We also test the most direct repair, a violation-contrastive objective that requires the algebra to fail on destroyed pairings: in the tested configurations it yields only a marginal separation within the reconstruction budget, on training and test triples alike. We recommend a validation protocol that these methods currently lack: a baseline-corrected metric, retraining on destroyed pairings, and a seed-budget analysis.
\end{abstract}

\section{Introduction}

Robot manipulation policies are limited by the cost of action-labelled trajectories, whereas videos of people manipulating objects are plentiful. Latent action models exploit the asymmetry~\cite{genie2024,lapa2025}: a relational encoder turns a pair of observations into a code for the transition between them, and the codes pretrain a vision-language-action policy~\cite{rt22023,openvla2024} grounded later on a few labelled demonstrations, with substantial reported gains~\cite{lapa2025,survey2026}.

A recent refinement adds algebraic structure~\cite{alam2026,aclam2026,rotvla2026,dlam2026}. If the latent transitions behave like displacements, then $a$ to $c$ should equal $a$ to $b$ composed with $b$ to $c$, and $b$ back to $a$ should undo $a$ to $b$. Two losses enforce these relations during pretraining, and two errors, an additivity error and a reversibility error, report how well they hold. Their reduction against an unstructured baseline, one to two orders of magnitude~\cite{alam2026}, is offered as evidence that the latent space represents transitions compositionally rather than arbitrarily. The Algebraically Consistent Latent Action Model (ALAM)~\cite{alam2026} reads this reduction, measured by what it calls algebraic probes, as showing that its latent space respects the identities and that the learned transitions generalise to held-out temporal spans. Three readings of such a reduction must be kept apart: that the codes satisfy the imposed constraint better, which the error measures directly; that they have captured which frame follows which; and that they carry the action content a downstream policy needs. We test the second and third, the readings that using the error as a label-free proxy for temporal structure presupposes. Neither follows from a low error, as the analysis below shows; composition itself is not in question, since a difference of state features composes by construction on any pairing.

\begin{figure}[t]
\centering
\includegraphics[width=0.9\linewidth]{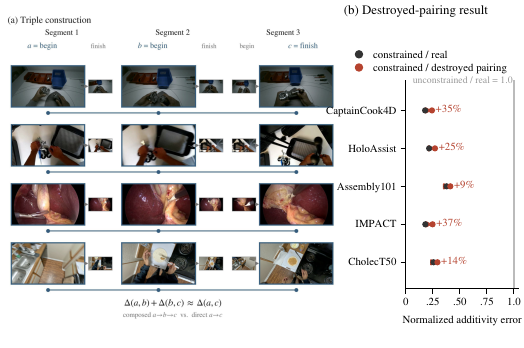}
\vskip-1em
\caption{\textbf{Low additivity error does not certify temporal pairing.} (a)~An additivity triple $a,b,c$ of real segment endpoints, on which a difference-of-features code satisfies $\Delta(a,b)+\Delta(b,c)\approx\Delta(a,c)$. (b)~Additivity error of the constrained model on real pairings (black) and retrained on destroyed pairings (red), per domain relative to the unconstrained real-pair model ($1.0$); red labels give the relative increase. The error rises but stays below the real-pair baseline.}
\label{fig:headline}
\vskip-1em
\end{figure}

These label-free errors are computable from video alone and inexpensive to report. A growing body of work reports them as evidence that a code has captured temporal structure~\cite{alam2026,aclam2026,rotvla2026,dlam2026}, up to a deployed mobility-and-manipulation system~\cite{abot2026}. Throughout, a \emph{certificate} is a low algebraic error taken, on its own, as evidence that temporal structure was captured. The certificate relies on the implication that a low error means temporal structure is present. Our question concerns evaluation practice across these methods, not any single system, and follows representation learning's reality checks~\cite{locatello2019,musgrave2020,henderson2018}.

The evidence does not support this implication. For the additive core, we show analytically that a model which reconstructs well is pushed towards representing the decoded transition as a difference of state features, for which both relations hold as identities regardless of the order in which the frames occurred. The errors can be small because the model reconstructs, not because it discovered anything about time, and the source-conditioned and quantised forms share the diagnosis empirically. Such a representation need not even be poor for control, since a difference of \emph{controllable} state is an action~\cite{whatmatters2026}, which is precisely the problem: the metric cannot tell it from a difference of nuisance state or an encoding convention.

Two measurements follow. Against an untrained anchor, an unconstrained model identical but for the algebraic losses already accounts for most of the reduction relative to that anchor, and the fold that survives against this trained counterpart depends as much on the decoder family as on what is learned. And because the identities hold for any pairing, we destroy the pairing of the pretraining frames within each dataset by successor swapping, replacing each successor while holding the source fixed so the annotated succession is removed (a \emph{destroyed pairing}), and retrain every arm: the constrained error rises, but remains below the same domain's unconstrained model on real pairings (Fig.~\ref{fig:headline}). A certificate whose passing level is attainable after the pairing is destroyed cannot be read as evidence that succession was captured. These errors are defined on the successor relation ALAM regularises~\cite{alam2026}, so the control targets succession; a variant composing over same-scene, non-successive triples~\cite{aclam2026} falls outside it.

To evaluate the policy implications, we re-instantiate the same objectives in a controlled, self-contained LIBERO~\cite{libero2023} pipeline whose pretrained encoder initialises the unmodified official policy. This is an independent policy-side test, not a direct transfer of any released system. Success does not follow the metric, and the constraint's benefit survives destroying the successor pairing it is credited with capturing.

This paper makes three contributions. First, a diagnosis: we prove and measure that a reconstructing additive model satisfies the algebraic identities as an artefact, and that the error's variation across triples follows the geometry of the feature chords rather than their pairing, so a low error cannot identify whether the code captured temporal succession. Second, a control: retraining every arm after destroying the temporal pairing, which none of the works reporting these metrics runs, shows the certified level is reached without it. Third, a protocol: a baseline-corrected metric, the destroyed-pairing control, a two-level action probe, and a seed-budget analysis, inexpensive enough for any method to report.

\section{Related Work}

\subsection{Latent Actions from Action-Free Video}

Genie~\cite{genie2024} showed that a code inferred between consecutive frames can stand in for an action well enough to drive an interactive world model; LAPO~\cite{lapo2024} made such codes a pretraining substrate for control; and LAPA~\cite{lapa2025} turned them into a pretraining target for vision-language-action policies~\cite{rt22023,openvla2024}, grounded afterwards on a few labelled demonstrations. This approach needs no action labels, complements action-supervised policy learning~\cite{diffusionpolicy2023}, and has become the standard route from human video to manipulation policy~\cite{survey2026}.

A recent line assumes these codes behave like displacements. ALAM~\cite{alam2026} regularises triplets so consecutive transitions sum to the one they span and a reversal cancels its forward counterpart; its quantised codes are decoded by a pixel decoder over source patches, and its $25$--$85\times$ reductions are ratios against a trained unconstrained model of that architecture, over rollout horizons. The Additively Compositional Latent Action Model (AC-LAM)~\cite{aclam2026} enforces additive composition over same-scene, not necessarily consecutive triples; RotVLA~\cite{rotvla2026} places latent actions in $SO(n)$; DLAM~\cite{dlam2026} makes them distributional; ABot-M0.5~\cite{abot2026} carries the losses into a deployed mobility-and-manipulation system. These five representative methods report an algebraic consistency error; none destroys its pretraining data's temporal pairing as a control, and none trains against violating negatives.

These methods address a different class of degenerate optima from the one we identify. AC-LAM observed that its inverse-dynamics form collapsed all latents to zero and adopted a forward form; RotVLA penalises latents that ``degenerate to trivial solutions''; AC-MTM~\cite{acmtm2026} adds an anti-collapse signal from action labels against the ``trivial solution of a constant encoder''. All are magnitude or constancy guards, caught by a variance floor. But the difference-of-features solution of Sec.~\ref{sec:analysis} keeps its scale, reconstructs well, and satisfies both identities for arbitrary pairings, so it passes them all.

\subsection{Diagnosing What the Codes Measure}

A complementary line asks what the codes contain and how to repair them: anchoring with a linear action head under distractors~\cite{laom2025}, tracing failures to exogenous state~\cite{whyfail2026}, filtering view-specific content by cross-view decoding~\cite{mvplam2026}, and a systematic design study that finds difference-like features competitive~\cite{whatmatters2026}. These repair or audit the codes' \emph{content}; our question is prior, whether the certifying error reads the data at all, and their anchoring is evaluated downstream in Sec.~\ref{sec:policy}. Zhang et al.~\cite{lamanalysis2025} relate the objective to principal component analysis in a linear setting; CD-LAM~\cite{cdlam2026} shows reconstruction-only training absorbs action-irrelevant factors; ATM~\cite{atm2026} probes latent transitions post hoc; Garrido et al.~\cite{wild2026} test whether an action inferred on one video transfers to another. These works diagnose the \emph{representation}. None asks whether the algebraic error is informative about the data it is computed on, and none destroys the temporal pairing of its own pretraining data as a control. Locatello et al.~\cite{locatello2019} found unsupervised disentanglement metrics neither identify their target without inductive bias nor predict downstream usefulness, the two failures we establish for the algebraic error; seed sensitivity in deep RL~\cite{henderson2018}, baseline-corrected metric learning~\cite{musgrave2020} and design-space studies of imitation learning~\cite{robomimic2021} are precedents of the same kind.

Contrastive objectives have also reached latent actions: ConLA~\cite{conla2026} contrasts a forward frame pair against its temporal reverse, CLAP~\cite{clap2026} contrasts vision-derived latents against ground-truth action tokens; both contrast \emph{features}. In contrast, Sec.~\ref{sec:method} contrasts the degree to which the algebra itself holds, with destroyed pairings as negatives. Temporal order verification is long established as a pretext task~\cite{shuffleandlearn2016}. We test the natural version of this idea, requiring the algebra to fail on disordered data, and find that in the tested configurations it yields only a marginal separation within the reconstruction budget: the algebra holds on disordered data as on ordered data. The survey~\cite{survey2026} cautions that structural biases ``do not guarantee that latent composition matches the physical composition'' of robot actions; our measurements provide the mechanism.

\section{Analysis of the Algebraic Errors}
\label{sec:analysis}

The two algebraic errors are normalised residuals of the additive and reversal relations,
\begin{equation}
  \mathcal{E}_{\mathrm{add}} = \frac{\lVert z_{ac} - (z_{ab} + z_{bc})\rVert}{\lVert z_{ac}\rVert}, \qquad
  \mathcal{E}_{\mathrm{rev}} = \frac{\lVert z_{ba} + z_{ab}\rVert}{\lVert z_{ab}\rVert},
  \label{eq:errors}
\end{equation}
with $z_{xy}$ the code the relational encoder assigns to the ordered pair $(x,y)$; a low value is read as evidence the code composes like a displacement. Eq.~\eqref{eq:errors} normalises each triple by its own transition norm and we summarise it by the median over triples; ALAM reports the unnormalised expectation $\mathbb{E}\lVert z_{ac}-(z_{ab}+z_{bc})\rVert$ and AC-LAM the ratio of expectations $\mathbb{E}\lVert z_{ac}-(z_{ab}+z_{bc})\rVert^2/\mathbb{E}\lVert z_{ab}\rVert^2$~\cite{alam2026,aclam2026}, and Sec.~\ref{sec:corrected} reports the conclusions under all three.

\subsection{Reconstruction Implies the Identities}

Write $\phi$ for the map from an observation to a state embedding and $z_{ab}$ for the code the relational encoder assigns to the ordered pair $(a,b)$. We analyse the additive core,
\begin{equation}
  \hat{\phi}(b) = \phi(a) + \mathrm{dec}(z_{ab}),
  \label{eq:additive}
\end{equation}
and train $\hat{\phi}(b)$ to reconstruct the target state embedding $\phi(b)$. Minimising this feature-space reconstruction residual therefore drives $\mathrm{dec}(z_{ab})$ towards $\phi(b)-\phi(a)$, and for any three frames
\begin{equation}
  \mathrm{dec}(z_{ac}) \approx \phi(c)-\phi(a)
  = \bigl(\phi(b)-\phi(a)\bigr) + \bigl(\phi(c)-\phi(b)\bigr),
  \label{eq:add-identity}
\end{equation}
whose right-hand side is $\mathrm{dec}(z_{ab}) + \mathrm{dec}(z_{bc})$. Reversibility follows the same way, from $\phi(a)-\phi(b) = -(\phi(b)-\phi(a))$.

The cancellation is algebraic: it uses neither $a<b<c$, nor that the frames share an episode, nor that anything happened between them. Any difference-of-features representation satisfies both relations, and reconstruction pushes the model into that family. The converse is classical: if additivity holds exactly on every triple through a fixed reference $o$, then $f(a,b)=g(b)-g(a)$ with $g(x)=f(o,x)$, the potential-difference family. The constraint selects it but cannot see what the potential carries: a difference of controllable state is an action, a difference of background state satisfies every identity equally well~\cite{whatmatters2026,whyfail2026}, and Eq.~\eqref{eq:errors} cannot tell them apart. The solution is trivial as evidence, not as a representation.

\begin{proposition}
\label{prop:injective}
Suppose reconstruction is accurate in the sense that $\lVert \mathrm{dec}(z_{xy}) - (\phi(y) - \phi(x)) \rVert \le \varepsilon$ for every pair used below. Then for any three frames and \emph{any} decoder,
\begin{gather*}
  \lVert \mathrm{dec}(z_{ac}) - \mathrm{dec}(z_{ab}) - \mathrm{dec}(z_{bc})
  \rVert \le 3\varepsilon,\\
  \lVert \mathrm{dec}(z_{ba}) + \mathrm{dec}(z_{ab}) \rVert \le 2\varepsilon,
\end{gather*}
and if $\mathrm{dec}$ is in addition linear with least singular value $\sigma_{\min} > 0$, the codes themselves obey $\lVert z_{ac} - (z_{ab} + z_{bc}) \rVert \le 3\varepsilon/\sigma_{\min}$ and $\lVert z_{ba} + z_{ab} \rVert \le 2\varepsilon/\sigma_{\min}$.
\end{proposition}

\emph{Proof.} $\mathrm{dec}(z_{ac}) - \mathrm{dec}(z_{ab}) - \mathrm{dec}(z_{bc})$ equals the residual of $(a,c)$ minus those of $(a,b)$ and $(b,c)$, because the state differences telescope exactly as in Eq.~\eqref{eq:add-identity}; its norm is therefore at most $3\varepsilon$, with no assumption on $\mathrm{dec}$. For linear $\mathrm{dec}$ of full column rank, $\lVert \mathrm{dec}(v) \rVert \ge \sigma_{\min}\lVert v \rVert$, which yields the code-space bounds; reversibility is the same argument on $\phi(a)-\phi(b) = -(\phi(b)-\phi(a))$. \hfill$\square$

The bounds concern the numerators of Eq.~\eqref{eq:errors}; normalisation can amplify a small residual when the transition norm is small. We clamp denominators at $10^{-6}$ for numerical stability and analyse this dependence empirically in Sec.~\ref{sec:injectivity}, which finds the variation of this per-triple normalised error across triples dominated by the norm rather than by the residual. Nothing in the statement or proof consults the order of $a$, $b$, $c$: reconstruction accuracy alone drives both identities to within the residual, for arbitrary pairings alike. The code-space bound needs the linear decoder: a nonlinear decoder can reconstruct exactly while its codes compose arbitrarily, so for the source-conditioned and quantised forms the diagnosis is empirical (Sec.~\ref{sec:architectures}).

\subsection{Reversibility as a Functional of the Encoder}
\label{sec:rev-blind}

Additivity involves three frames and a composition, so it is at least possible for it to depend on which frame follows which. Reversibility involves no third frame.

\begin{proposition}
\label{prop:antisymmetry}
$\mathcal{E}_{\mathrm{rev}}$ is a functional of the relational encoder alone. If $f$ is antisymmetric, $f(x,y) = -f(y,x)$, then $\mathcal{E}_{\mathrm{rev}}(x,y) = 0$ for \emph{every} pair of inputs $(x,y)$, whether or not $y$ follows $x$, whether or not they come from the same episode, and whether or not either is a video frame.
\end{proposition}

The proof is immediate from the definition: $\mathcal{E}_{\mathrm{rev}}$ in Eq.~\eqref{eq:errors} is built from $f(a,b)$ and $f(b,a)$ and takes no third frame, so no successor relation enters it. Antisymmetric $f$ are not exotic; $f(x,y) = g(y) - g(x)$ is one, and it is exactly the family that reconstruction drives the model towards.

This has two consequences. First, exact antisymmetry gives zero error on every input, so a low reversibility error cannot on its own certify true succession. Second, because $\mathcal{E}_{\mathrm{rev}}$ reads only the pair $(a,b)$, any intervention that leaves that pair intact cannot move it for a fixed model. So where reversibility differs in Table~\ref{tab:shuffle}, whose destroyed columns are evaluated on the same pre-state pairs as the real ones, two \emph{models} differ, not two readings of one.

\subsection{The Baseline Reference}

Since Eq.~\eqref{eq:add-identity} already holds for a reconstructing model carrying no algebraic loss, that model is the reference against which a constrained model's error must be read; an untrained network supplies an absolute anchor for how large the error can be. ALAM's own ratios are indeed taken against a trained unconstrained model~\cite{alam2026}, and the decomposition below shows why such ratios take the magnitude they do. Writing $\mathcal{E}_{\mathrm{rand}}$, $\mathcal{E}_{\mathrm{rec}}$ and $\mathcal{E}_{\mathrm{alg}}$ for the three errors,
\begin{equation}
  \rho_{\mathrm{raw}} =
  \frac{\mathcal{E}_{\mathrm{rand}}}{\mathcal{E}_{\mathrm{alg}}},
  \qquad
  \rho_{\mathrm{corr}} =
  \frac{\mathcal{E}_{\mathrm{rec}}}{\mathcal{E}_{\mathrm{alg}}},
  \label{eq:corrected}
\end{equation}
and the share of the reduction the constraint is responsible for is $(\mathcal{E}_{\mathrm{rec}}-\mathcal{E}_{\mathrm{alg}})/ (\mathcal{E}_{\mathrm{rand}}-\mathcal{E}_{\mathrm{alg}})$.

\section{Experiments}

\subsection{Experimental Setup}
\label{sec:setup}

Transitions are learned from five source domains chosen to differ in viewpoint, segment length and subject matter: instructional kitchen video~\cite{captaincook2024}, egocentric assistance~\cite{holoassist2023}, multi-view assembly~\cite{assembly1012022}, egocentric industrial assembly~\cite{impact2026}, and laparoscopic surgery~\cite{cholect2022}, on frozen V-JEPA~2.1 features~\cite{vjepa2025}. Recordings are split three ways, and hyperparameters are selected per architecture on validation reconstruction while every reported error is computed on the untouched test split unless a training-set quantity is named. Arms cross the presence of the constraint with the integrity of the temporal pairing, three seeds each; destroyed-pairing arms are additionally repeated over three independent pairing realisations.

The anchor $\mathcal{E}_{\mathrm{rand}}$ is the same error on an untrained transition model, averaged over three initialisations. Errors are computed as per-triple normalised ratios (Eq.~\eqref{eq:errors}, $10^{-6}$ in the denominator), summarised by their median over triples and averaged over seeds. A variance hinge on the state embedding guards the normalisation. Transition models are two-layer maps (AdamW~\cite{adamw2019}, $3{,}000$ steps, batch $256$).

Policy experiments use LIBERO~\cite{libero2023}, re-instantiating the same objectives in a self-contained pipeline: a small convolutional state encoder is trained from scratch on the benchmark's own demonstration frames (six-channel stacks of the two camera views), under the identical reconstruction, constraint, violation-contrastive and anchoring losses, with no frozen video features. Each arm's encoder then initialises the visual encoder of the benchmark's behaviour-cloning transformer, otherwise unmodified: one three-channel ResNet~\cite{resnet2016} per camera from the same checkpoint, all shape-compatible backbone weights transferred ($29$ of $32$ tensors) while the dimension-mismatched input convolution and projection head are reinitialised, and the policy fine-tuned per task under the unmodified protocol.

\subsection{Verifying the Analysis}
\label{sec:injectivity}

We test the premises of Proposition~\ref{prop:injective} on the models of Table~\ref{tab:corrected}, over the $4{,}334$ test triples of the five domains, both additive arms and three training seeds. The per-pair residual $\varepsilon$ is $1$--$10\%$ of the transition norm with the constraint and $3$--$24\%$ without, and the unseen two-step pair $(a,c)$ has no larger a residual than the one-step pairs in any trained run. Every per-triple bound holds, the measured numerator at $0.15$--$0.39$ of it: reconstruction accuracy alone caps the unconstrained model's decoded-space error at a median $0.02$--$0.09$ against an untrained anchor near $1$. The decoder has full rank $64$ but condition number $6$ to $2.2\times10^{4}$, so the code-space constant is weak where conditioning is poor; the errors on $z$ and $\mathrm{dec}(z)$ differ by $27\%$ at the median and $84\%$ at worst, and the ordering of the constrained and unconstrained models is unchanged in every domain and seed. What matters is not that the bound holds but that it is small because $\varepsilon$ is: reconstruction accuracy, not any temporal property, caps the error.

Proposition~\ref{prop:injective} bounds the residual, not the ratio the line reports, and Fig.~\ref{fig:geometry} separates the two on every test triple. Feature chords are far from collinear (median bend angle $92$--$119^{\circ}$), yet between the straightest and the most bent quintile of triples the constrained residual $\lVert z_{ac}-z_{ab}-z_{bc}\rVert$ changes by at most $1.3\times$ while the normalising norm $\lVert z_{ac}\rVert$ falls to $0.43$--$0.99$ of its value, so the reported error rises by $1.2$--$2.8\times$; reversibility behaves the same way in the displacement (residual within $1.5\times$, normaliser growing $3.2$--$11\times$). The same dependence holds for the model retrained on destroyed pairings, whose curve sits above the real-pair model by the offset of Table~\ref{tab:shuffle}, and within those pairings the additivity error shows no detectable monotonic relation to how far the swapped successor lies from its origin (rank correlation within $\pm0.10$ in every domain): what varies across triples is the geometry of the chords, to which Eq.~\eqref{eq:add-identity} is indifferent.

We directly demonstrate that the errors cannot identify the information encoded by a feature difference. In a simulated world with a controllable coordinate $q$ and an independent nuisance coordinate $n$, we construct three exact difference representations: of $q$, of $n$, and of both, their concatenation rescaled by $1/\sqrt{2}$. Over three seeds all three score at machine precision on both true and destroyed pairings ($\mathcal{E}_{\mathrm{add}}\approx10^{-16}$, $\mathcal{E}_{\mathrm{rev}}=0$), while action decodability is $R^2=1.00$ for the controllable difference, $0.00$ for the nuisance one, and $1.00$ for their even mixture: any procedure certifying by these errors certifies all three equally. Compensated reparameterisations $z\to Az$ change both errors by under $2\%$ on both arms, so this non-identifiability is not specific to one latent coordinate system.

\begin{figure}[t]
\centering
\includegraphics[width=\linewidth]{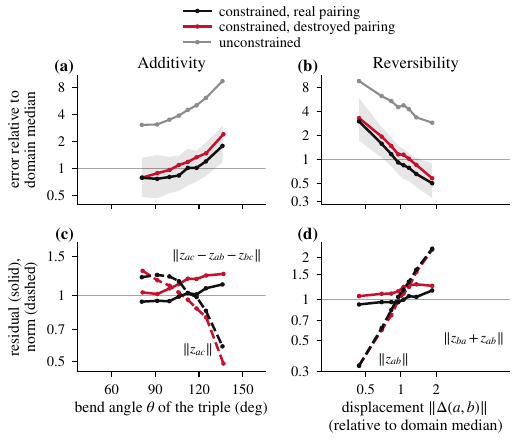}
\vskip -1em
\caption{\textbf{Variation of the per-triple normalised error is dominated by its normaliser.} Test triples of five domains binned by bend angle (left) or displacement $\lVert\Delta(a,b)\rVert$ (right), where $\Delta(a,b)=\phi(b)-\phi(a)$. (a,b)~Binned median errors, shaded by the interquartile range of the constrained real-pair model. (c,d)~Their residual numerators (solid) and denominator norms (dashed). Vertical quantities are relative to the constrained-real domain median and the displacement to the domain median displacement.}
\label{fig:geometry}
\end{figure}

\subsection{Baseline-Corrected Reduction}
\label{sec:corrected}

Table~\ref{tab:corrected} compares the two fold reductions. Relative to the untrained anchor, \emph{any} reconstructing model accounts for most of the error reduction, without an algebraic loss. A released model shows the same insensitivity to the pairing: on LAPA's released model~\cite{lapa2025}, quantised and source-conditioned, the errors on robot triples are $1.03$ and $1.91$, and destroying the temporal pairing leaves them at $1.02$ and $1.96$. The same holds under the metric definitions of ALAM and AC-LAM: recomputed on the same test triples, ALAM's unnormalised expectation gives folds of $1.0$--$4.2\times$ and AC-LAM's ratio of expectations $8.6$--$32\times$, against $2.6$--$5.6\times$ for Eq.~\eqref{eq:errors}; the constrained model retrained on destroyed pairings rises by $0$--$23\%$ and $7$--$75\%$ respectively, and stays below the unconstrained real-pair model in every domain under the ratio of expectations and in four of five under the unnormalised expectation, the exception being Assembly101 ($5.28$ against $5.14$).

\begin{table}[t]
\centering
\caption{Fold reduction against the untrained anchor ($\rho_{\mathrm{raw}}$) and the trained counterpart ($\rho_{\mathrm{corr}}$), and the constraint's share of the total reduction.}
\label{tab:corrected}
\vskip -0.5em
\resizebox{0.9\columnwidth}{!}{%
\begin{tabular}{llrrrrr}
\toprule
Domain & Error & $\mathcal{E}_{\mathrm{rec}}$ & $\mathcal{E}_{\mathrm{alg}}$ & $\rho_{\mathrm{raw}}$ & $\rho_{\mathrm{corr}}$ & Share \\
\midrule
CaptainCook4D~\cite{captaincook2024} & Add. & 0.084 & 0.015 & 69.3$\times$ & 5.6$\times$ & 6.7\% \\
CaptainCook4D & Rev. & 0.100 & 0.013 & 129.8$\times$ & 7.9$\times$ & 5.3\% \\
HoloAssist~\cite{holoassist2023} & Add. & 0.054 & 0.012 & 85.8$\times$ & 4.6$\times$ & 4.3\% \\
HoloAssist & Rev. & 0.088 & 0.022 & 87.9$\times$ & 4.0$\times$ & 3.5\% \\
Assembly101~\cite{assembly1012022} & Add. & 0.100 & 0.038 & 26.5$\times$ & 2.6$\times$ & 6.5\% \\
Assembly101 & Rev. & 0.123 & 0.047 & 40.3$\times$ & 2.6$\times$ & 4.0\% \\
IMPACT~\cite{impact2026} & Add. & 0.042 & 0.008 & 131.2$\times$ & 5.5$\times$ & 3.5\% \\
IMPACT & Rev. & 0.053 & 0.010 & 174.6$\times$ & 5.4$\times$ & 2.5\% \\
CholecT50~\cite{cholect2022} & Add. & 0.110 & 0.028 & 35.5$\times$ & 3.9$\times$ & 8.3\% \\
CholecT50 & Rev. & 0.419 & 0.097 & 20.4$\times$ & 4.3$\times$ & 17.1\% \\
\bottomrule
\end{tabular}}
\vspace{-0.5em}
\end{table}

\subsection{Sensitivity to Temporal Pairing}
\label{sec:shuffle}

Within each recording we replace every post-state with a different post-state from the same recording, chosen to approximately match the original feature-displacement magnitude, while holding the pre-state and the split fixed; the true successor is never retained. An additivity triple $(a,b,c)$ takes $a$ and $b$ from the pre-states of consecutive segments and $c$ from a post-state, so the swap changes $c$ and leaves the pair $(a,b)$ on which reversibility is evaluated intact. Every arm is retrained from scratch and evaluated on the destroyed pairings.

Across the five domains, $72$--$99\%$ of the replaced successors carry a different action label from the original, and the replacement lies a median of $1$--$7$ segments away. Yet the feature difference the errors are computed on carries almost no cue of the destruction: a linear classifier trained to separate real from destroyed pairs reaches balanced accuracy $0.49$--$0.50$, chance level, and a one-hidden-layer multilayer perceptron (MLP) recovers only a weak nonlinear cue ($0.52$--$0.61$). The control thus removes the annotated succession while leaving the feature difference nearly uninformative about it.

Table~\ref{tab:shuffle} compares the four combinations of constraint and pairing conditions. A response exists, but its magnitude is not calibrated to temporal validity. Under the constraint additivity rises in all five domains, by $+9\%$ to $+37\%$ and in the same direction in every seed and every pairing realisation, yet the constrained model retrained on destroyed pairings still reaches $0.010$--$0.041$, below the unconstrained model on that domain's \emph{real} pairings in every case. Outside that one consistent cell the response has no consistent direction: the unconstrained counterpart's additivity moves by $-4\%$ to $+38\%$, reversibility by $-26\%$ to $+33\%$ with mixed sign, and fifteen of the twenty cells rise while five fall. The inconsistency is present with and without the algebraic losses, so it is not produced by them.

The control separates two axes that a single comparison would conflate. For a fixed model, which triples are evaluated is not the question: reversibility is evaluated on the pair $(a,b)$, which the swap leaves intact, so the evaluation pairing cannot move it for a fixed model, and additivity is bounded by the reconstruction residual on whatever triples are scored (Proposition~\ref{prop:injective}). Scoring every trained model on both pairings makes this distinction explicit. The constrained model trained on real pairings scores $0.008$--$0.038$ on real test triples and $0.009$--$0.040$ on destroyed ones; the destroyed-trained model scores $0.009$--$0.041$ and $0.010$--$0.041$, respectively. At a fixed evaluation pairing, changing the training pairing moves the constrained additivity error by $+2\%$ to $+35\%$; at a fixed training pairing, changing the evaluation pairing moves it by at most $12\%$, or $0.003$ in absolute terms, and leaves reversibility exactly unchanged. Every cell of the cross nevertheless remains below the unconstrained model on real pairings. Retraining therefore isolates the other axis, what the model learned from real as against destroyed pairings; the destroyed-trained model's excess persists across the whole range of chord geometry (Fig.~\ref{fig:geometry}), and it is on that axis that the certified level is reached without the pairing.

The triples above keep one real edge, $a\to b$ between consecutive pre-states. A second control builds the triple from the model's own transition pairs, $(a,b,c)=(\mathrm{pre}_i,\mathrm{post}_i,\mathrm{post}_{i+1})$, for loss and evaluation alike: on destroyed pairings $a\to b$ is the swapped pair, and a real neighbour survives on $b\to c$ and $a\to c$ only where magnitude matching selects an adjacent segment, in $5$--$17\%$ and $4$--$38\%$ of triples (the upper figures on CholecT50, whose swaps lie a median of one segment away). On these triples the constrained model retrained on destroyed pairings rises by $+5\%$ to $+63\%$ and remains below the unconstrained real-pair model of the same protocol in every domain, at $0.010$--$0.045$ against $0.046$--$0.138$.

\begin{table}[t]
\centering
\caption{Algebraic error with and without the constraint on real and destroyed pairings; the constrained additivity rise holds in every seed and realisation.}
\label{tab:shuffle}
\vskip -0.5em
\resizebox{0.9\columnwidth}{!}{%
\begin{tabular}{lrrrrrrrr}
\toprule
& \multicolumn{4}{c}{Additivity} & \multicolumn{4}{c}{Reversibility} \\
\cmidrule(lr){2-5}\cmidrule(lr){6-9}
& \multicolumn{2}{c}{Recon.} & \multicolumn{2}{c}{+Alg.}
& \multicolumn{2}{c}{Recon.} & \multicolumn{2}{c}{+Alg.} \\
Domain & real & destroyed & real & destroyed & real & destroyed & real & destroyed \\
\midrule
CaptainCook4D~\cite{captaincook2024} & 0.084 & 0.084 & 0.015 & 0.020 & 0.100 & 0.093 & 0.013 & 0.019 \\
HoloAssist~\cite{holoassist2023} & 0.054 & 0.074 & 0.012 & 0.015 & 0.088 & 0.117 & 0.022 & 0.027 \\
Assembly101~\cite{assembly1012022} & 0.100 & 0.109 & 0.038 & 0.041 & 0.123 & 0.148 & 0.047 & 0.050 \\
IMPACT~\cite{impact2026} & 0.042 & 0.053 & 0.008 & 0.010 & 0.053 & 0.067 & 0.010 & 0.012 \\
CholecT50~\cite{cholect2022} & 0.110 & 0.105 & 0.028 & 0.032 & 0.419 & 0.312 & 0.097 & 0.096 \\
\bottomrule
\end{tabular}}
\vspace{-0.5em}
\end{table}

\subsection{Decoder Architectures}
\label{sec:architectures}

The derivation covers additive decoders but not two alternatives in use: a decoder conditioned on the source observation without adding the transition to it, and one that quantises the transition before decoding. Table~\ref{tab:arch} reports all three; its additive rows are the domain means of Table~\ref{tab:shuffle}. The quantised arm is evaluated on its hard discrete code, the quantity such a method constrains and deploys; on destroyed pairings the median limits the influence of the few two-step pairs with nearly coincident endpoints, which drive a triple-mean to diverge.

Under the constraint all three decoders reach a low additivity error, $0.020$, $0.023$, $0.134$, and destroying the pairing leaves them low, at $0.024$, $0.025$, $0.177$; constrained reversibility stays below $0.05$ throughout. The architecture determines the reported reduction factor: the same constraint on the same data yields $3.9\times$ with an additive decoder and $32.9\times$ with a source-conditioned one, because the latter's unconstrained error is $0.768$ rather than $0.078$. The fold reflects the degree to which the unconstrained architecture violates an identity it was not trained to satisfy, and is not comparable across families. Under ALAM's unnormalised expectation and AC-LAM's ratio of expectations the folds become $1.8\times$ and $15\times$ (additive), $20\times$ and $1{,}152\times$ (source-conditioned) and $8.8\times$ and $11\times$ (quantised), and, averaged over domains, under every definition and decoder the constrained model retrained on destroyed pairings remains below the unconstrained real-pair model. That ratio factorises per domain into the fall of the squared residual and the change in code energy $\mathbb{E}\lVert z_{ab}\rVert^2$: for the source-conditioned decoder the residual falls $21$--$3{,}859\times$ while the energy changes $0.8$--$25\times$, for the additive decoder $1$--$15\times$ against $2$--$8\times$, so the fold reflects joint changes in residual and code-energy statistics rather than a direct measure of temporal structure.

\begin{table}[t]
\centering
\caption{Algebraic error under three decoder forms, mean over domains, seeds and realisations of the per-run median; the quantised arm is evaluated on its hard code. Fold $=$ unconstrained\,/\,constrained on real pairings.}
\label{tab:arch}
\vskip -0.5em
\resizebox{0.9\columnwidth}{!}{%
\begin{tabular}{lrrrrr}
\toprule
& \multicolumn{2}{c}{Unconstrained} & \multicolumn{2}{c}{Constrained} & \\
\cmidrule(lr){2-3}\cmidrule(lr){4-5}
Decoder & real & destroyed & real & destroyed & fold \\
\midrule
\multicolumn{6}{l}{\emph{Additivity}} \\
Additive & 0.078 & 0.085 & 0.020 & 0.024 & $3.9\times$ \\
Source-conditioned~\cite{alam2026} & 0.768 & 0.764 & 0.023 & 0.025 & $32.9\times$ \\
Quantised~\cite{lapa2025} & 0.924 & 0.899 & 0.134 & 0.177 & $6.9\times$ \\
\midrule
\multicolumn{6}{l}{\emph{Reversibility}} \\
Additive & 0.156 & 0.147 & 0.038 & 0.041 & $4.1\times$ \\
Source-conditioned & 1.536 & 1.497 & 0.039 & 0.036 & $39.6\times$ \\
Quantised & 0.272 & 0.258 & 0.014 & 0.016 & $19.2\times$ \\
\bottomrule
\end{tabular}}
\vspace{-0.5em}
\end{table}

\subsection{Violation-Contrastive Repair}
\label{sec:method}

The most direct repair is to demand a difference: the difference-of-features solution passes the collapse guards of Sec.~II but scores identically on real and destroyed pairings, so an objective that requires the algebra to fail on destroyed pairings cannot be satisfied by it. Let $\mathcal{E}(\cdot)$ denote the errors of Eq.~\eqref{eq:errors} on a batch of triples, $\mathcal{T}$ triples in their true order and $\tilde{\mathcal{T}}$ triples re-drawn so that no original annotated successor edge is retained:
\begin{equation}
  \mathcal{L} = \mathcal{L}_{\mathrm{rec}}
  + \lambda \Bigl[ \mathcal{E}(\mathcal{T})
  + \max\bigl(0,\; m - \mathcal{E}(\tilde{\mathcal{T}})\bigr) \Bigr].
  \label{eq:objective}
\end{equation}
The first bracketed term is the usual constraint; the second requires the destroyed pairings to violate it by a margin $m$. The negatives are drawn from a pairing realisation independent of the one used for evaluation, and we evaluate all twelve $(\lambda,m)$ configurations.

Separation from these negatives does not by itself isolate temporal order: the negatives differ from true successor pairs in the relation under test but also in how the pairs were sampled, so a separation alone does not apportion credit between succession and construction. Matching negatives for episode, endpoints and step size removes the nuisance differences we can name, and a reconstruction-only baseline measures the separation that appears with no order-sensitive term.

Separation is the difference between the median additivity error on destroyed and on real triples. Across twelve $(\lambda,m)$ configurations ($\lambda\in\{0.03,0.1,0.3,1.0\}$, $m\in\{0.1,0.3,1.0\}$) and all five domains, in the $27$ of $60$ domain-configuration pairs whose reconstruction stays within a factor of two of the unconstrained model, the separation is at most $0.04$ on training triples, including a pairing realisation the model never trained against, and at most $0.03$ on test triples, against a reconstruction-only baseline of up to $0.01$; the constraint alone separates by $0.02$ at most. Configurations that separate by more than $0.04$ reach $0.08$ at best and pay for it with $3$--$17\times$ the unconstrained reconstruction error. Negatives matched for recording, $(a,b)$ pair and step magnitude are separated by $0.02$ at training time and by less than $0.003$ on test. The gap is consistent with the pairing-insensitive difference-of-features solution of Sec.~\ref{sec:analysis}.

The hinge term averages the normalised error over a batch, and a triple-mean of the same models reports destroyed-triple errors near $20$, driven by the few destroyed two-step pairs whose normalising norm approaches zero (the divergence of Sec.~\ref{sec:architectures}), while the median gap stays small; a separation carried by a few near-zero normalisers could not certify order in any case. Semantic anchoring, a small amount of action supervision reaching the encoder~\cite{laom2025}, selects among these algebraically indistinguishable solutions by what they predict; Sec.~\ref{sec:policy} evaluates it.

\subsection{Downstream Policy Evaluation}
\label{sec:policy}

The benchmark policy with random visual initialisation reaches $0.870$ over the ten tasks and serves as the reference. Each pretrained arm is evaluated with the same policy architecture, optimisation and rollout protocol, transferring all shape-compatible visual parameters, and the arms are compared with one another as a dissociation test. Every pretrained initialisation performs below the reference, and on LIBERO-GOAL destroying the pairing moves either arm by under one percentage point, so no pairing-dependent difference is resolved at the present seed budget.

The constraint's contribution is more limited than the certificate implies. It adds $4.0 \pm 4.6$ points on real data and $2.8$ on data whose successor pairing was destroyed, each positive in ten of the thirteen seeds. Whether the gain \emph{depends} on the pairing is a weaker question: the difference between the two gains averages $+1.1$ points against a seed spread of seven (positive in five of thirteen), and most of the benefit persists without the structure it is credited to. The repair adds $+0.4 \pm 3.6$ points over the constraint (positive in eight of thirteen seeds), no resolved downstream improvement at the present seed budget.

For the anchoring comparison, a linear head predicting the demonstrated action sequence over the pretraining gap reaches $0.821 \pm 0.030$, a paired gain of $+3.8 \pm 4.1$ points over the unconstrained arm against the constraint's $+4.0 \pm 4.6$, and adding the constraint yields $-0.2 \pm 3.6$ points (positive in six of thirteen seeds).

\begin{table}[t]
\centering
\caption{Policy success, mean $\pm$ SD over $13$ seeds of the per-seed task average ($50$ episodes per task), by pretraining objective and pretraining-data pairing. Top: LIBERO-GOAL against the random-init reference; bottom: LIBERO-SPATIAL.}
\label{tab:policy}
\footnotesize
\begin{tabular}{llc}
\toprule
Pretraining objective & Pairing & Success $\uparrow$ \\
\midrule
\multicolumn{3}{l}{\emph{LIBERO-GOAL}} \\
Benchmark default (random init) & -- & $0.870$ \\
Unconstrained & Real & $0.783 \pm 0.038$ \\
Unconstrained & Destroyed & $0.792 \pm 0.031$ \\
Algebraic constraint & Real & $0.822 \pm 0.027$ \\
Algebraic constraint & Destroyed & $0.820 \pm 0.029$ \\
Violation-contrastive & Real & $0.827 \pm 0.024$ \\
Label anchoring~\cite{laom2025} & Real & $0.821 \pm 0.030$ \\
Constraint $+$ anchoring & Real & $0.819 \pm 0.017$ \\
\midrule
\multicolumn{3}{l}{\emph{LIBERO-SPATIAL}} \\
Unconstrained & Real & $0.739 \pm 0.047$ \\
Unconstrained & Destroyed & $0.744 \pm 0.043$ \\
Algebraic constraint & Real & $0.718 \pm 0.034$ \\
Algebraic constraint & Destroyed & $0.773 \pm 0.036$ \\
\bottomrule
\end{tabular}
\vspace{-0.5em}
\end{table}

The second benchmark sharpens this result (Table~\ref{tab:policy}, lower). On LIBERO-SPATIAL, destroying the pairing changes the unconstrained arm by only $+0.006$, but raises the constrained arm by $+0.054 \pm 0.043$, positive in all thirteen matched seeds. The constraint's four-point advantage on real LIBERO-GOAL data does not reappear either: on real SPATIAL data its effect is $-0.020$. These results do not establish that destroyed pairing is generally beneficial; they show that the observed policy differences cannot be attributed to preserving the successor relation.

The pipeline allows us to compare the algebraic certificate and representation content on the same arms. On demonstrations no arm trained on, the additivity error is $0.92$ without the constraint and $0.07$ with it; anchoring alone gives $0.99$, and combining anchoring with the constraint gives $0.007$. Action content, by a ridge probe from the transition code to the demonstrated action sequence over the pretraining gap, trained on pairs from the pretraining demonstrations and scored by pooled multivariate $R^2$ on pairs from $25$ further demonstrations per task that no arm used for any loss, identical pairs for every arm (five seeds): from the transition code, the demonstrated action sequence is decodable at $R^2=0.49\pm0.02$ (unconstrained), $0.41\pm0.03$ (constrained), $0.50\pm0.02$ (anchoring), $0.22\pm0.04$ (both); from the transferred state embeddings the arms are level ($0.62$--$0.63$, each SD $\le0.03$). The certificate runs against content: across these four arms the mean decodability falls as the certificate improves. Paired over the five matched seeds, the constraint lowers code decodability by $0.07 \pm 0.04$ without anchoring and by $0.27 \pm 0.05$ with it, negative in every seed, while linear action decodability from the transferred state embeddings changes by $0.00 \pm 0.01$. The measured effect is therefore larger at the code level than at the state level; this dissociation alone does not explain the policy differences. Success is monotonic in neither certificate nor decodability; the comparison does not identify what produces it.

Table~\ref{tab:dissociation} reports the full results. The best- and worst-certified arms, constraint$+$anchor ($\mathcal{E}_{\mathrm{add}}=0.007$) and anchoring ($0.99$), differ by $0.2$ points downstream, within the seed spread, while their code decodability runs the opposite way ($0.22$ against $0.50$): selecting encoders by the algebraic error would pick the representation with the least linearly decodable action under this probe.

\begin{table}[t]
\centering
\caption{Certificate (five seeds) against linearly decodable action content ($R^2$, ridge probe), both on demonstrations unused by any pretraining loss, from the transition code and the transferred states, with $13$-seed policy success; certificate and probe extrema in bold.}
\label{tab:dissociation}
\vskip -0.5em
\resizebox{0.9\columnwidth}{!}{%
\begin{tabular}{lrrrrr}
\toprule
& \multicolumn{2}{c}{Certificate} & \multicolumn{2}{c}{Action $R^2$} & \\
\cmidrule(lr){2-3}\cmidrule(lr){4-5}
Arm & $\mathcal{E}_{\mathrm{add}}\downarrow$ & $\mathcal{E}_{\mathrm{rev}}\downarrow$ & code$\uparrow$ & states$\uparrow$ & Success$\uparrow$ \\
\midrule
Unconstrained & $0.923$ & $1.924$ & $0.486$ & $0.624$ & $0.783$ \\
Constraint & $0.074$ & $0.223$ & $0.412$ & $0.624$ & $0.822$ \\
Anchoring~\cite{laom2025} & $0.985$ & $1.980$ & $\mathbf{0.496}$ & $\mathbf{0.634}$ & $0.821$ \\
Constraint\,$+$\,anchor & $\mathbf{0.007}$ & $\mathbf{0.025}$ & $0.224$ & $\mathbf{0.634}$ & $0.819$ \\
\bottomrule
\end{tabular}}
\end{table}

\subsection{Robustness}
\label{sec:seedbudget}

Recomputing each difference over every choice of three of our thirteen seeds, a three-seed experiment could have reported the four-point effect anywhere from $-2.6$ to $+10.1$ points, the wrong sign to $2.5\times$ the full-budget estimate, the seed sensitivity documented in deep reinforcement learning~\cite{henderson2018}.

The domain results above all use frozen V-JEPA~2.1 features; we re-run the constrained and unconstrained models on four of the five domains with three further encoders: VideoMAEv2~\cite{videomae2023} fine-tuned on Kinetics-710 (appearance-dominated) and on SSv2~\cite{ssv22017} (temporally supervised), and the single-frame DINOv2~\cite{dinov22024}. Per encoder and domain the corrected fold stays in single digits ($2.9$--$5.9\times$), and retraining on destroyed pairings moves the constrained error by $-2\%$ to $+31\%$, upward in ten of the twelve encoder-domain pairs, while it remains below the unconstrained real-pair error in every one. Neither explicit temporal supervision nor a single-frame backbone changes the picture: the metric requires no temporally computed features and cannot certify them.

\section{Conclusion}

Additivity and reversibility errors are widely read as evidence of compositional temporal structure; a low error can arise from reconstruction alone, without identifying temporal succession or action-relevant content. The decoded transitions of a reconstructing additive decoder satisfy both identities to within the reconstruction residual whatever the frames' pairing, a zero reversibility error certifies nothing about succession, the per-triple normalised error's variation across triples is dominated by its normaliser, the reported order-of-magnitude reductions depend on the decoder family and on the metric definition as much as on what is learned, and the certified level is reached after the temporal pairing is destroyed, on four encoders, while a released model's errors are as insensitive to the pairing as ours. Downstream, success is not monotonic in the certificate, the code's mean linear action decodability falls as the certificate improves across the tested arms, and a small amount of action supervision yields a comparable mean downstream gain in the tested pipeline. These claims concern what the algebraic error certifies about a representation, not the policy performance of any system. What the line needs is a way to tell a certificate that tracks temporal structure from one that does not: report $\rho_{\mathrm{corr}}$ (Eq.~\ref{eq:corrected}) against a trained, architecture-matched counterpart rather than $\rho_{\mathrm{raw}}$ against an untrained one; retrain under destroyed pairings and report whether the constrained model still improves over the same real-pair reference, since an improvement that persists cannot certify preserved succession; and report per-seed sign counts at a budget that resolves the effect. These checks test the reading that a low error certifies succession, not every diagnostic of pairing: failing them rules out that reading, and passing them motivates downstream validation rather than replacing it.

\section*{Acknowledgment}
The project is funded by the Deutsche Forschungsgemeinschaft (DFG, German Research Foundation) -- SFB-1574 -- 471687386. This work was supported in part by the SmartAge project sponsored by the Carl Zeiss Stiftung (P2019-01-003; 2021-2026). The authors gratefully acknowledge the computing time provided on the high-performance computer HoreKa by the National High-Performance Computing Center at KIT (NHR@KIT). This center is jointly supported by the Federal Ministry of Education and Research and the Ministry of Science, Research and the Arts of Baden-W\"urttemberg, as part of the National High-Performance Computing (NHR) joint funding program (\url{https://www.nhr-verein.de/en/our-partners}). HoreKa is partly funded by the German Research Foundation (DFG).

\bibliographystyle{IEEEtran}
\bibliography{reference}
\end{document}